%% file: Main-Conference.tex
\documentclass[table,conference]{IEEEtran}
\IEEEoverridecommandlockouts
\usepackage{cite}
\usepackage{amsmath,amssymb,amsfonts}
\usepackage{algorithmic}
\usepackage{graphicx}
\usepackage{textcomp}
\usepackage{ntheorem}
\usepackage{booktabs}
\usepackage{hyperref}[colorlinks,linkcolor=blue]
\usepackage{tablefootnote}
\newtheorem{proposition}{\textbf{Proposition}}[section]

\newtheorem{proof}{\textbf{Proof}}[section]
\theorembodyfont{\it}
\usepackage{multirow}
\usepackage{booktabs}
\usepackage{subfigure}
\usepackage{diagbox}
\usepackage{wrapfig}
\usepackage{makecell}
\usepackage[ruled,linesnumbered]{algorithm2e}

\usepackage{tikz}
\usetikzlibrary{calc}

\newcommand*\blackcircled[1]{\tikz[baseline=(char.base)]{
        \node[shape=circle,fill={rgb,255:red,0;green,0;blue,0}, text=white, font=\small, inner sep=0.6pt] (char) {#1};}}

\def\BibTeX{{\rm B\kern-.05em{\sc i\kern-.025em b}\kern-.08em
    T\kern-.1667em\lower.7ex\hbox{E}\kern-.125emX}}
\begin{document}

\title{Rethinking Efficiency and Redundancy in \\ Training Large-scale Graphs
}


\author{\IEEEauthorblockN{Xin Liu\IEEEauthorrefmark{1}\IEEEauthorrefmark{2},
Xunbin Xiong\thanks{1 Xin Liu and Xunbin Xiong contribute equally.}\IEEEauthorrefmark{3}, 
Mingyu Yan\thanks{2 Mingyu Yan is the corresponding author.}\IEEEauthorrefmark{1}\IEEEauthorrefmark{2}, 
Runzhen Xue\IEEEauthorrefmark{1}\IEEEauthorrefmark{2}, 
Shirui Pan\IEEEauthorrefmark{4}, 
Xiaochun Ye\IEEEauthorrefmark{1}\IEEEauthorrefmark{2},
and
Dongrui Fan\IEEEauthorrefmark{1}\IEEEauthorrefmark{2}}
\IEEEauthorblockA{\IEEEauthorrefmark{1}SKLP, Institute of Computing Technology, Chinese Academy of Sciences, China\\
\IEEEauthorrefmark{2}University of Chinese Academy of Sciences, China\\
\IEEEauthorrefmark{3}School of Information Science and Technology, ShanghaiTech University, China\\
\IEEEauthorrefmark{4}School of ICT, Griffith University, Australia}
}

\maketitle

\begin{abstract}

Large-scale graphs are ubiquitous in real-world scenarios and can be trained by Graph Neural Networks (GNNs) to generate representation for downstream tasks. Given the abundant information and complex topology of a large-scale graph, we argue that redundancy exists in such graphs and will degrade the training efficiency. Unfortunately, the model scalability severely restricts the efficiency of training large-scale graphs via vanilla GNNs. Despite recent advances in sampling-based training methods, sampling-based GNNs generally overlook the redundancy issue. It still takes intolerable time to train these models on large-scale graphs. Thereby, we propose to drop redundancy and improve efficiency of training large-scale graphs with GNNs, by rethinking the inherent characteristics in a graph.

In this paper, we pioneer to propose a once-for-all method, termed DropReef, to drop the redundancy in large-scale graphs. Specifically, we first conduct preliminary experiments to explore potential redundancy in large-scale graphs. Next, we present a metric to quantify the neighbor heterophily of all nodes in a graph. Based on both experimental and theoretical analysis, we reveal the redundancy in a large-scale graph, i.e., nodes with high neighbor heterophily and a great number of neighbors. Then, we propose DropReef to detect and drop the redundancy in large-scale graphs once and for all, helping reduce the training time while ensuring no sacrifice in the model accuracy. To demonstrate the effectiveness of DropReef, we apply it to recent state-of-the-art sampling-based GNNs for training large-scale graphs, owing to the high precision of such models. With DropReef leveraged, the training efficiency of models can be greatly promoted. DropReef is highly compatible and is offline performed, benefiting the state-of-the-art sampling-based GNNs in the present and future to a significant extent.

\end{abstract}


\input{samples/1-intro}
\input{samples/2-bg}
\input{samples/3-metric}
\input{samples/5-exp}

\input{samples/6-relatedW}
\input{samples/7-conclusion}

\bibliographystyle{./IEEEtran}
\bibliography{./Reference}

\end{document}

%% file: samples/1-intro.tex
\section{Introduction}

Learning graph data has become a hot spot of recent research and has drawn increasing attention in the deep learning domain. To capture information from complex graphs, various methods are abundantly proposed, among which Graph Neural Networks (GNNs) \cite{scarselli2008graph} are superior exemplars for tackling diverse tasks \cite{graphsage,gnn_app2,gnn_app4,gnn_app_security}. Owing to the outstanding performance and explainability of GNN \cite{pope2019explainability,zhang2022trustworthy}, recent years have witnessed the emergence of diverse GNN variants, e.g., Graph Convolutional Networks (GCNs) \cite{kipf2016semi}, Graph Attention Networks (GATs) \cite{velickovic2018graph}, and Graph Isomorphism Networks (GINs) \cite{xu2018how}. 

However, it generally takes nontrivial efforts to train a well-expressive GNN.
Conventionally, training of GNNs, especially GCNs, is generally performed in a full-batch manner \cite{clustergcn}, which introduces undesirable training efficiency and storage consumption \cite{kipf2016semi}. Moreover, as the scale of real-world graph data rapidly grows by the day, the conventional training method cannot even work properly because of the ``out of memory'' (OOM) issue. Thereby, efforts have been made to improve the training efficiency of GNNs \cite{survey1,survey2}. Notably, sampling-based GNN models \cite{graphsage,fastgcn,asgcn,LADIES,RWT,clustergcn,graphsaint,AC-sampling,MV-sampling,Bandit-sampling,Biased-sampling,PGSampling} are emphatically studied and focused on by researchers, owing to their great capacity of training graphs in an efficient manner. Taking a graph as the input, sampling-based models utilize a well-designed sampler to acquire a part of the whole graph by step, i.e., a subgraph, and train the subgraphs in a mini-batch manner, instead of directly performing training on the original one. Overall, sampling-based training methods ensure an acceptable convergence rate and the training efficiency.

Unfortunately, the ever-growing graph scale has been challenging the efficiency of the sampling-based training method. Previously, typical graph-related tasks such as node classification are always conducted on small datasets \cite{smalldataset}, 
resulting in a limited power in terms of model scalability. Recent literature proposes to train well-expressive models on large-scale graphs \cite{clustergcn,graphsaint}. Nevertheless, it takes hundreds or thousands of times longer on a large graph to yield minor promotion in terms of accuracy than on a small one.
In addition, previous investigations \cite{Pagraph,survey3} have discovered that the sampling process is becoming a bottleneck to the training of GNNs, which further results in low-efficiency training of large-scale graphs. 
Thereby, a question arises: \textit{What exactly degrades \textbf{efficiency} of training large-scale graphs}?

We answer that: \textit{\textbf{Redundancy} in large-scale graphs can slow down the training of GNNs}. We further argue that, \textit{\textbf{dropping} \textbf{redundancy} in large-scale graphs can significantly \textbf{increase the efficiency} of training large-scale graphs with GNNs}. 
In the graph-related domain, previous literature \cite{robustGNN1,redundancy-GNN1,robustGNN2,zhang2022understanding,robustGNN3,robustGNN4} proposes to detect and reduce redundant information (also regarded as noise) from a graph to promote the model performance or yield a robust model for facing adversarial attacks. However, none of a large-scale graph is considered or applied for analysis, owing to its inherent complexity. In addition, most of the literature \cite{robustGNN1,robustGNN2,robustGNN3} pays close attention to the model accuracy and omits the potential of improving efficiency. 
Thus, we propose to accelerate the training of large-scale graphs with GNNs by dropping redundancy. 
Moreover, considering the diversity of GNN variants, we believe that a well-designed technique with high compatibility can benefit state-of-the-art models in the present and future better, instead of designing a completely new GNN. Thereby, all we need is a once-for-all method to detect and drop redundancy in large-scale graphs while ensuring no sacrifice in the model accuracy.

In this paper, we pioneer to propose \textbf{DropReef}, a novel method to \underline{\textbf{Drop}} the \underline{\textbf{Re}}dundancy in large-scale graphs onc\underline{\textbf{e}} and \underline{\textbf{f}}or all, helping improve the efficiency of training large-scale graphs with GNNs. 
Specifically, we first conduct preliminary experiments to explore potential redundancy in large-scale graphs. Next, we propose to quantify the neighbor heterophily of all nodes in a graph by a novel metric, termed weighted neighbor heterophily (WNH), according to the fact that nodes with a heterophilic neighboring distribution are hard to classify accurately by vanilla GNNs. Based on the analysis, we argue that redundancy in a large-scale graph is nodes with high WNH and a great number of neighbors. Finally, we present DropReef to detect and drop the redundancy in large-scale graphs once and for all. Nodes with high WNH and a great number of neighbors are regarded as redundant nodes and are dropped from the training set, thus yielding a significant acceleration in training with no sacrifice in the model accuracy. 

Contributions of this paper are summarized as follows:
\begin{itemize}
\item We first conduct preliminary experiments to quantify the neighbor distribution and explore the potential \textit{\textbf{redundancy}} (i.e., information that is of no benefit to the model training) in large-scale graphs. Experimental results indicate two findings: the neighbor distribution of nodes is generally imbalanced, and redundancy exists in node regions that are densely connected. Then, we present a novel metric, termed weighted neighbor heterophily, to quantify the neighbor heterophily for nodes in a graph. Based on both experimental and theoretical analysis, we reveal the \textit{\textbf{redundancy}} in a large-scale graph, i.e., nodes with high neighbor heterophily and a great number of neighbors. 

\item We propose DropReef, a once-for-all method to drop the \textbf{\textit{redundancy}} in large-scale graphs and improve the \textit{\textbf{efficiency}} of training large-scale graphs with GNNs. DropReef is easy to use and can be offline performed. One can directly train the large-scale graphs that have been processed by DropReef to witness a considerable acceleration in terms of GNNs training straightforwardly.

\item We conduct experiments on four real-world graphs in large scale and apply DropReef to three state-of-the-art sampling-based GNNs. Experimental results have demonstrated the effectiveness of DropReef in terms of efficiency and accuracy. DropReef is highly compatible and helps reduce on average 26.80\% and up to 83.01\% of the training time among four large-scale graphs while ensuring no sacrifice in the model accuracy.
\end{itemize}

%% file: samples/2-bg.tex
\section{Preliminary and Motivation}
In this section, we first introduce the fundamental of GNNs and sampling-based training methods. Next, we make an observation on the topology of large-scale graphs and argue that the potential redundancy heavily exists in such graphs. Then, we put forward our motivation and highlight the opportunity to improve the training efficiency, revealed by preliminary experiments.

\subsection{Fundamental of GNNs}

GNNs are prevalent deep learning models for handling graph-related tasks \cite{survey4}. Studies \cite{gnn_app3,graphsage,xu2018how,graphsaint,gnn_prediction} have previously revealed the superior performance of GNNs on tasks of classification and prediction. In general, GNN models utilize the same computing pattern to learn the hidden information from a graph, despite diverse variants of GNNs. Herein, we take a widely used model, i.e., GCN \cite{kipf2016semi}, as an exemplar for the formulated introduction. Given the adjacency matrix \textbf{A} and the initial feature matrix \textbf{H}$^0$ as the input, GCN captures the hidden representation from a graph and propagates the information by layers in an iterative manner. The forward propagation is formulated as follows:
\begin{equation} \label{Eq1} 
    \textbf{H}^{l} = \sigma \left( \widetilde{\textbf{D}}^{-\frac{1}{2}} \widetilde{\textbf{A}} \widetilde{\textbf{D}}^{-\frac{1}{2}} \textbf{H}^{l-1} \textbf{W}^{l-1} \right)
\end{equation}
where $\widetilde{\textbf{A}}$ denotes the normalized adjacency matrix, and $\widetilde{\textbf{D}}$ denotes the degree matrix of $\widetilde{\textbf{A}}$. $\sigma$ denotes a nonlinear activation function. \textbf{W}$^{l-1}$ and \textbf{H}$^{l-1}$ denote the learnable weight matrix and the hidden feature matrix in the (\textit{l}-1)-th layer, respectively. To sum up, training of GNNs consists of the forward propagation for learning the representation (as given in Equation \eqref{Eq1}) and the backward propagation for updating the model parameters (i.e., \textbf{W}).

\subsection{Sampling-based Training Methods}
Learning the representation from a large-scale graph takes nontrivial cost in terms of computation and storage under the propagation pattern given in Equation \eqref{Eq1} since the hidden feature \textbf{H} in each layer is computed based on $\widetilde{\textbf{A}}$ of the whole huge graph. 
Thereby, sampling-based training methods \cite{graphsage,fastgcn,asgcn,LADIES,RWT,clustergcn,graphsaint,AC-sampling,MV-sampling,Bandit-sampling,Biased-sampling,PGSampling} are proposed to diminish the original large graph to smaller ones and feed them to a GNN in a batched manner, thus reducing the computation and storage cost of GNN training.
Classical sampling-based training methods propose to select nodes from all neighbors of one node in a random manner \cite{graphsage}, or select nodes from one node batch according to the pre-computed sampling distribution \cite{fastgcn}. 

\begin{figure*}[t]
\centering
\includegraphics[width=\textwidth]{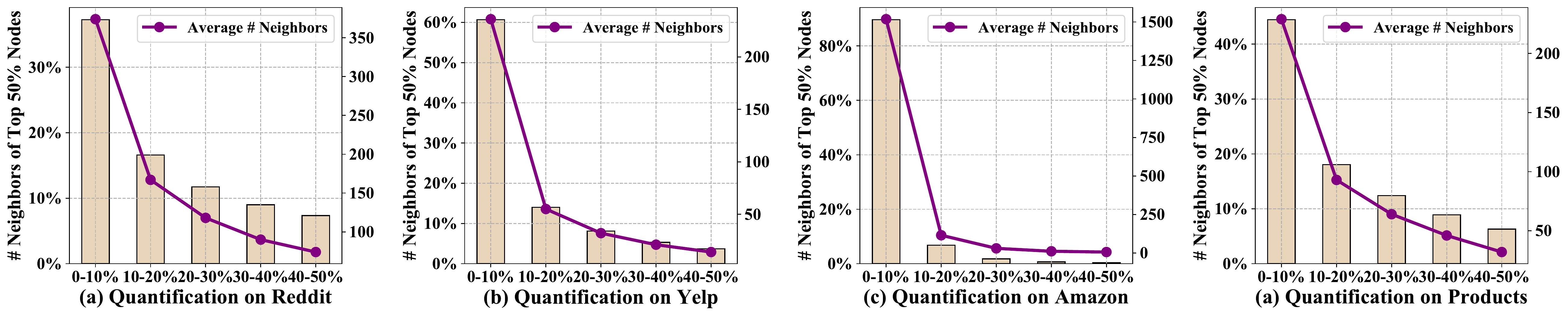}
\caption{Quantification on four large-scale graphs. Please note that we sort the nodes by their neighbors' amount and pick the \textit{top 50\%} ones for analysis. In each subplot, the selected \textit{top 50\%} of nodes are divided into quintiles. For example, ``0-10\%'' in the \textbf{x-axis} denotes the top quintile of nodes in the \textit{top 50\%} of nodes. Bars in each subplot are the ratio of the number of neighbors (abbr. \# Neighbors) of the \textit{top 50\%} of nodes to the number of neighbors of total nodes in a graph. The curves in each subplot are the quantification of the average number of neighbors of the \textit{top 50\%} of nodes.}
\label{fig1}
\end{figure*}

Nevertheless, such classical methods treat each node identically during sampling and omit the correlation between nodes being sampled in general. Moreover, experiments in such methods are generally conducted on small graphs.
As the improvement, recent literature \cite{clustergcn,PGSampling,graphsaint} proposes to extend GNN training to larger graphs and deeper model structures. They sample subgraphs and construct mini-batches based on the sampled ones. Training of GNN is subsequently conducted on these mini-batches. Given an un-directed graph $\mathcal{G}(\mathcal{V}, \mathcal{E})$, the above process can be formulated as follows:
\begin{equation}
    Node~Set:~N \leftarrow \texttt{\textbf{Sample}}(\mathcal{V}, budget)
\end{equation}
\begin{equation}
    Subgraph:~\mathcal{G}_s(\mathcal{V}_s, \mathcal{E}_s) \leftarrow \texttt{\textbf{Construct}}(N_1, N_2, \cdots)
\end{equation}
where $N$ denotes a node set that is sampled from the original graph $\mathcal{G}$. The size of $N$ is restricted to $budget$. Then, in each mini-batch, a subgraph $\mathcal{G}_s$ is constructed based on the sampled node sets and is fed for subsequent training. These subgraph-based sampling methods are generally superior in accuracy since they consider the correlation between nodes when constructing a subgraph \cite{graphsaint}. The sampled subgraphs are dense and have many frequently accessed nodes (i.e., nodes with a lot of neighbors), which, however, can introduce redundancy in sampling as well as training of GNNs.


\subsection{Observations on Potential Redundancy}

Large-scale graphs are complex in terms of graph topology, making the performance of GNN on such graphs distinct from small ones \cite{OGB}.
Specifically, in large-scale graphs, dense local communities \cite{chen2013detecting} are more easily formed by node regions in which nodes are densely connected to other nodes in a neighboring region, compared to small graphs. In such regions, nodes, especially in a central position, share a great number of edges with other neighboring nodes \cite{clauset2005finding}, yielding many high-degree nodes in large-scale graphs. 

\noindent \textbf{Definition 1} \textit{High-degree nodes (abbr. HD nodes) are nodes with a great number of direct neighbors (one-hop neighbors).}

Generally, abundant information exists in such large-scale graphs, and so does the redundancy (i.e., information that is of no benefit to the model training). HD nodes in large-scale graphs result in an imbalanced neighbor distribution. 
Since HD nodes are sampled with a high probability, it takes nontrivial cost to construct subgraphs based on HD nodes and their neighbors, given their densely connected patterns. Therefore, we propose to analyze such densely connected nodes and quantify the potential redundancy.


We quantify the overall neighbor distribution in large-scale graphs, i.e., Reddit \cite{graphsage}, Yelp \cite{graphsaint}, Amazon \cite{graphsaint}, and OGBN-Products (abbr. Products) \cite{OGB}. Detailed information about datasets is given in Table \ref{tab1}.
We first sort nodes in a graph by their neighbors' amount. Then, we quantify the ratio of neighbors' amount of the HD nodes, more specifically, the \textit{top 50\%} of nodes with a great number of neighbors, to the neighbors' amount of the total nodes in a graph.  
Please note that total neighbors' amount of all nodes in a graph generally far exceeds the total number of nodes in a statistical sense, meaning that different nodes might have a great number of the same neighbors. The quantification allows us to analyze node regions that are possibly dense in a fine-grained manner. As illustrated in Figure \ref{fig1}, the neighbors' amount of the \textit{top 10\%} of nodes dominates a considerable ratio in all large-scale graphs, and is appreciably higher than nodes in other parts. The \textit{top 10\%} of nodes are HD nodes, concluding that the distribution of the node connected density is imbalanced in a graph and is centralized in regions composed of HD nodes.
These observations indicate that subgraphs constructed by such HD nodes are generally dense. We suppose that such dense subgraphs have a higher probability of containing redundancy and bring about greater overhead than ones in a sparse pattern. Thereby, our supposition can be proved if the model accuracy does not decline significantly after dropping HD nodes.

\noindent \textbf{Remark} In this paper, only the one-hop neighbors of a node are considered. 

\begin{figure}[t]
\centering
\includegraphics[width=0.49\textwidth]{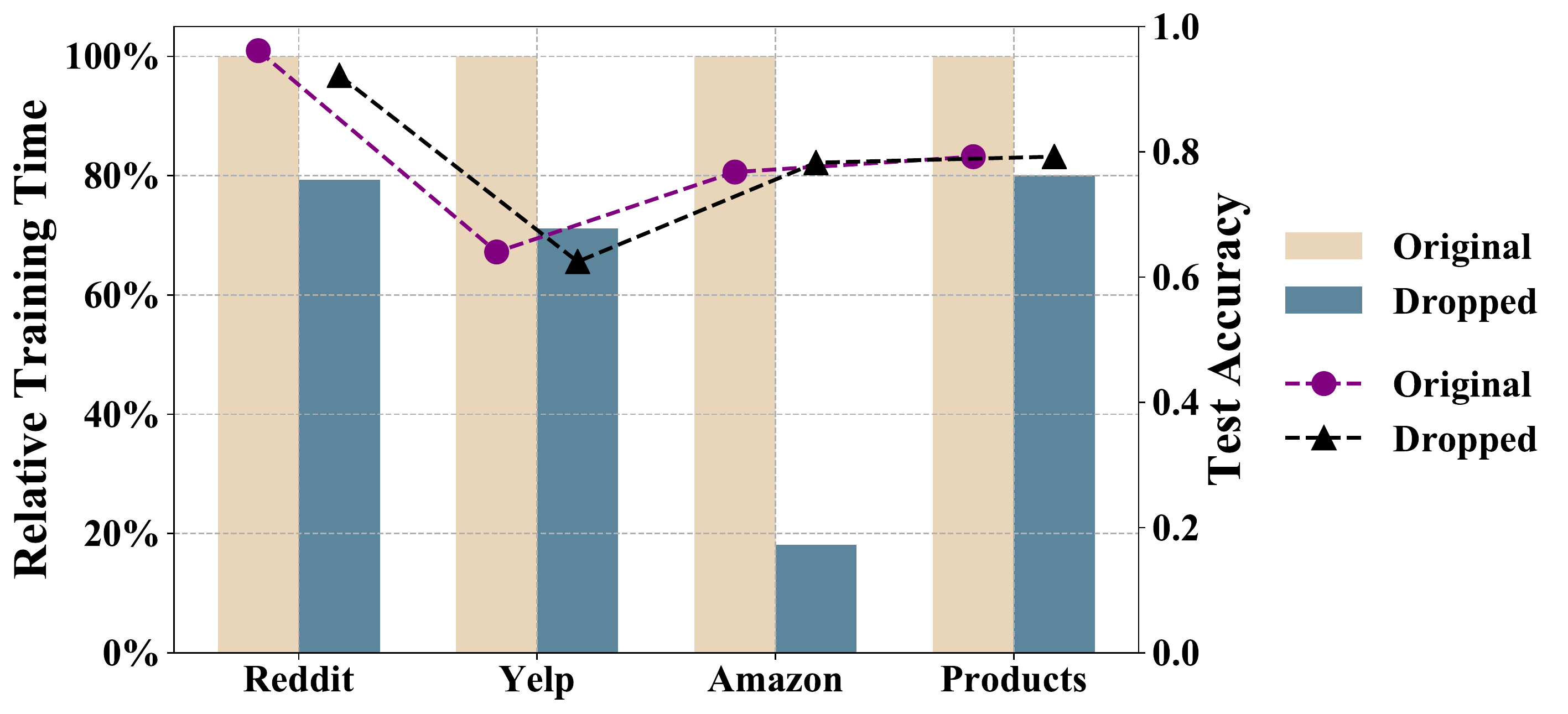}
\caption{Comparison on the training time and the test accuracy among four large-scale graphs, with GraphSAINT used as the backbone. Please note that the training process consists of sampling and pure model training processes. The sampling process is accomplished by a random node sampler. Bars denote the relative training time, while dotted lines denote the variation of the test accuracy. \textbf{Original} denotes the model is normally trained, while \textbf{Dropped} denotes that a naive drop method is performed on datasets before they are fed to the model. Training configurations between \textbf{Original} and \textbf{Dropped} are completely consistent to the official configurations in GraphSAINT.}
\label{fig2}
\end{figure}

\subsection{Opportunities to Improve Efficiency} \label{sec:2.3}

Preliminary experiments are conducted to highlight the influence of dropping HD nodes on four large-scale graphs. Naively, we drop the \textit{top 5\%} of training nodes on each dataset before they are fed to the model for training. We compare the \textbf{training time} and the \textbf{test accuracy} between the original model training and training with the naive drop method equipped. The only difference is that a part of HD nodes are dropped in datasets in the latter case. As illustrated in Figure \ref{fig2}, we find that the acceleration of the training time is nontrivial. Dropping HD nodes reduced 80\% of the training time on Amazon dataset, and average 25\% of training time on the other three datasets. Nevertheless, the significant reduction in the training time does not bring about a dramatic decline in the model accuracy. For example, the accuracy decline ratio on Reddit dataset is 4.12\%.

According to the preliminary experiments, dropping a part of HD nodes from the total training nodes can generally reduce the time cost of training a large-scale graph, especially for graphs in which HD nodes concentrate in the \textit{top 10\%} of nodes (e.g., Yelp and Amazon datasets). In addition, the decline in the model accuracy is slight on all datasets, which can be attributed to the fact that \textbf{the dropped HD nodes contain redundant information to a large extent}. Motivated by the above findings, we argue that designing a more precise drop method to remove HD nodes containing redundant information can attach both efficient training and trivial sacrifice in model accuracy, instead of naively dropping the \textit{top 5\%} of nodes. And it will be best to have this method performed once and for all, avoiding the occupation of online computation and storage resources. Therefore, all we need is an accurate drop method that can be offline performed to detect and decrease redundancy in large-scale graphs.

%% file: samples/3-metric.tex
\section{Methodology}
In this section, we first define a metric, termed weighted neighbor heterophily, to quantify the neighbor heterophily for nodes in a graph. Based on the proposed metric, we define the data redundancy and redundant nodes in a graph accordingly. Next, we propose DropReef, a once-for-all method to detect and drop the redundant nodes in large-scale graphs. DropReef consists of three subprocesses and is performed offline. After processing by DropReef, the training efficiency on large-scale graphs is considerably promoted. 

\subsection{The Proposed Metric}
\subsubsection{Neighbor Heterophily} This subsubsection introduces a key component of the metric, i.e., neighbor heterophily.

In real-world graphs, homophily is a basic principle to describe the distribution of nodes' classes that connected nodes within a region generally belong to the same class \cite{mcpherson2001birds}. A straightforward example that a person shares common interests and habits with his friends can well describe this ``birds of a feather flock together'' phenomenon \cite{Networks}. We further derive a connected pattern for a neighbor region from this phenomenon. As illustrated in Figure \ref{fig3}(a), a homophilic graph is obtained by treating a person and his friends (with the same interests and habits) as nodes in the same class. In another case, given a node \textit{v} that is densely connected with its neighbors, labels of \textit{v}'s neighbors are diverse in which major of the neighbors' classes is different from the class to which \textit{v} belongs. The above connected pattern, which is given in Figure \ref{fig3}(b), is an unusual yet possible case in an unrestricted graph, despite the homophily principle. 
However, many existing GNNs, e.g., GCN \cite{kipf2016semi} and LGNN \cite{LGNN}, are designed under a strong assumption of homophily \cite{beyond_homo}, and trends to overfit the majority classes, rendering undesirable learning representation on minority ones \cite{MulC-imbalanced}. Therefore, supposing that a GNN is trained for node classification, \textbf{node \textit{v} is harder to be accurately classified than the node \textit{m}} (given their connected patterns in Figure \ref{fig3}).



\begin{figure}[t]
\centering
\includegraphics[width=0.45\textwidth]{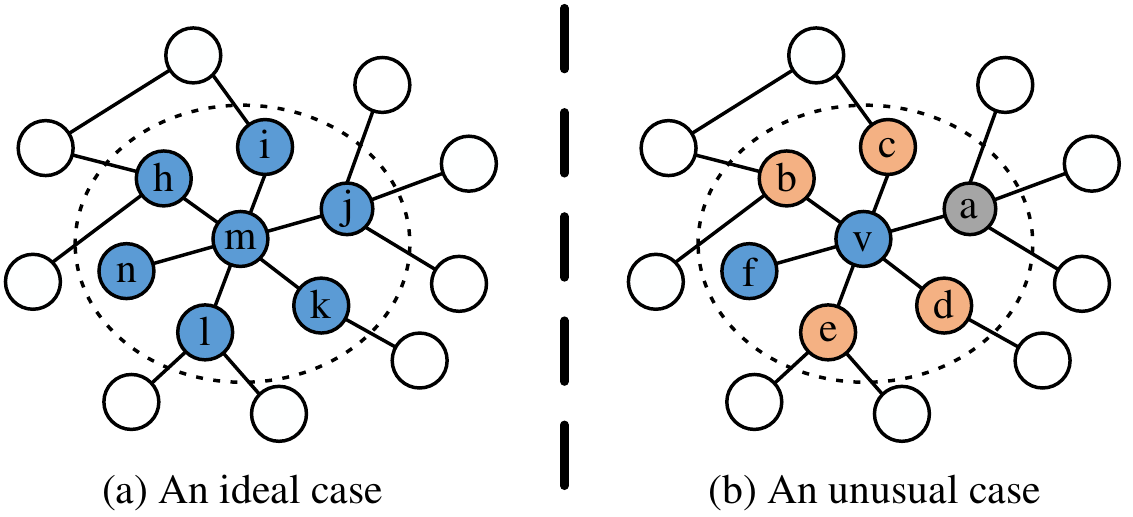}
\caption{An illustration of two distributions of nodes' classes. Please note that colors represent classes, and nodes in the same color belong to the same class.}
\label{fig3}
\end{figure}

\noindent \textbf{Definition 2} \textit{Neighbor heterophily is measurable value to reflect the heterophily for a node within its neighboring region.}

We argue that nodes with heterophilic neighbors can hardly be of benefit to the model performance. And such nodes can be detected by calculating the difference between them and their neighbors. 
Recent literature \cite{beyond_homo} presented a metric to empirically reflect the edge homophily ratio for a whole graph from a global perspective. In our case, we propose to calculate the difference between a node and its neighbors from a local perspective by quantifying the neighbor heterophily within its neighboring region, considering that nodes with high neighbor heterophily are in a fairly smaller proportion than normal ones.
The neighbor heterophily $Hete_{v}$ for a given node \textit{v} can be quantified as follows: 
\begin{equation} \label{Eq4}
    Hete_{v} = \frac{1}{D_v}\sum_{u \in N(v)}\Vert \mathbf{c}_v - \mathbf{c}_u\Vert_2
\end{equation}
where $D_v$ and $N(v)$ denote the degree and the neighbor set of node $v$, respectively. ${\mathbf{c}_v \in \mathbb{R}^{N_c}}$ denotes a label vector extended by the class to which node $v$ belongs, where $N_c$ denotes the number of all possible classes of nodes.

\begin{proposition} \label{prop1}
For node $v$, $Hete_v$ is associated with the distribution of nodes' classes in $v$'s neighboring region. A higher $Hete_v$ indicates a more heterophilic neighboring distribution of $v$. In the task of node classification, the upper bound of $Hete_v$ in single-class classification is $\sqrt{2}$ and in multi-class classification is $\sqrt{N_c}$.
\end{proposition}
\begin{proof}
For single-class node prediction, assuming that node $v$ and its neighbor $u$ belong to different classes, we can simply derive the 2-Norm of ($\mathbf{c}_v - \mathbf{c}_u$):
\begin{equation} \label{Eq5}
    \Vert \mathbf{c}_v - \mathbf{c}_u\Vert_2 = \sqrt{2}
\end{equation}
Next, we denote \textbf{Diff} as the number of v's neighbors that belong to different classes from v, and calculate $Hete_v$ according to Equation \eqref{Eq4}:
\begin{equation} \label{Eq6}
    \begin{split}
    Hete_{v} &= \frac{1}{D_v}\sum_{u \in N(v)}\Vert \mathbf{c}_v - \mathbf{c}_u\Vert_2 \\
    &= \sqrt{2} \frac{\textbf{Diff}}{D_v}
    \end{split}
\end{equation}
Based on Equation \eqref{Eq6}, $Hete_v$ is positively correlated with \textbf{Diff} since D$_v$ is a constant when the node v is given. Therefore, a more heterophilic neighboring distribution of v indicates a larger \textbf{Diff}, thus boosting $Hete_v$ accordingly. It is also observed that the upper bound of $Hete_v$ is $\sqrt{2}$ since $\frac{\textbf{Diff}}{D_v}$ will not exceed 1.

Proof in a multi-class situation is same to the single-class one. Assuming that node $v$ and its neighbor $u$ belong to different classes, we denote \textbf{Max} as the maximum of $\Vert \mathbf{c}_v - \mathbf{c}_u\Vert_2$, and give the upper bound of $Hete_v$ as follows:
\begin{align}
    Hete_{v} &= \frac{1}{D_v}\sum_{u \in N(v)}\Vert \mathbf{c}_v - \mathbf{c}_u\Vert_2 \le \textbf{Max} \frac{\textbf{Diff}}{D_v}
\end{align}
Possible maximums of \textbf{Max} and \textbf{Diff} can be simultaneously acquired in one particular case: Any two of v and its neighbors do not belong to the same class. In this case, we can derive the value of \textbf{Max} as $\sqrt{N_c}$ and \textbf{Diff} as D$_v$, thus inferring the upper bound of $Hete_v$ as $\sqrt{N_c}$.

\end{proof}

\subsubsection{Linking Probability} This subsubsection introduces a key component of the metric, i.e., linking probability.

In social networks, link prediction is a task to predict the existence of the edge between two entity nodes. In this task, a classifier is learned to mark edges with true or false labels \cite{LinkPredSurvey,LinkPred_1}. In the GNN-related domain, a GNN model is trained to yield linking probabilities associated with edges by ingesting features and the adjacency matrix of the target graph \cite{LinkPred_2,VGAE}. The generated linking probabilities are regarded as the existence probabilities for unobserved edges between nodes in general. Considering a case of sampling subgraphs, a designed sampler iteratively samples nodes before the nodes' amount exceeds the budget. Subsequently, a subgraph is constructed by recovering existing edges between the sampled nodes. Therefore, edges are unobserved for the sampler before the subgraph is built accordingly.
Inspired by the weighted message passing in GATs \cite{velickovic2018graph}, we propose to add linking probabilities to $Hete_v$ since each neighbor has a unique impact on node $v$. A demonstration will be given in Section \ref{sec:3.1.3} to reveal that $Hete_v$ with linking probabilities added is more explicable than a naive $Hete_v$ for exceptive cases.

\subsubsection{Definition of the Metric} \label{sec:3.1.3}

We propose a metric, termed weighted neighbor heterophily (WNH), to reflect the neighbor heterophily for a given node v with all individual edges between v and its neighbors considered simultaneously. The proposed metric is calculated as follows:

\begin{equation} \label{Eq8}
    WNH_{v} = \frac{1}{D_v}\sum_{u \in N(v)} p_{vu} \cdot \Vert \mathbf{c}_v - \mathbf{c}_u\Vert_2
\end{equation}
where $p_{vu}$ denotes the linking probability between nodes $v$ and $u$. As the improvement of the naive neighbor heterophily given in Equation \eqref{Eq4}, we add linking probabilities served as the weight into the proposed metric, making the proposed metric more explicable than a naive one for exceptive cases.

\begin{figure}[t]
\centering
\includegraphics[width=0.5\textwidth]{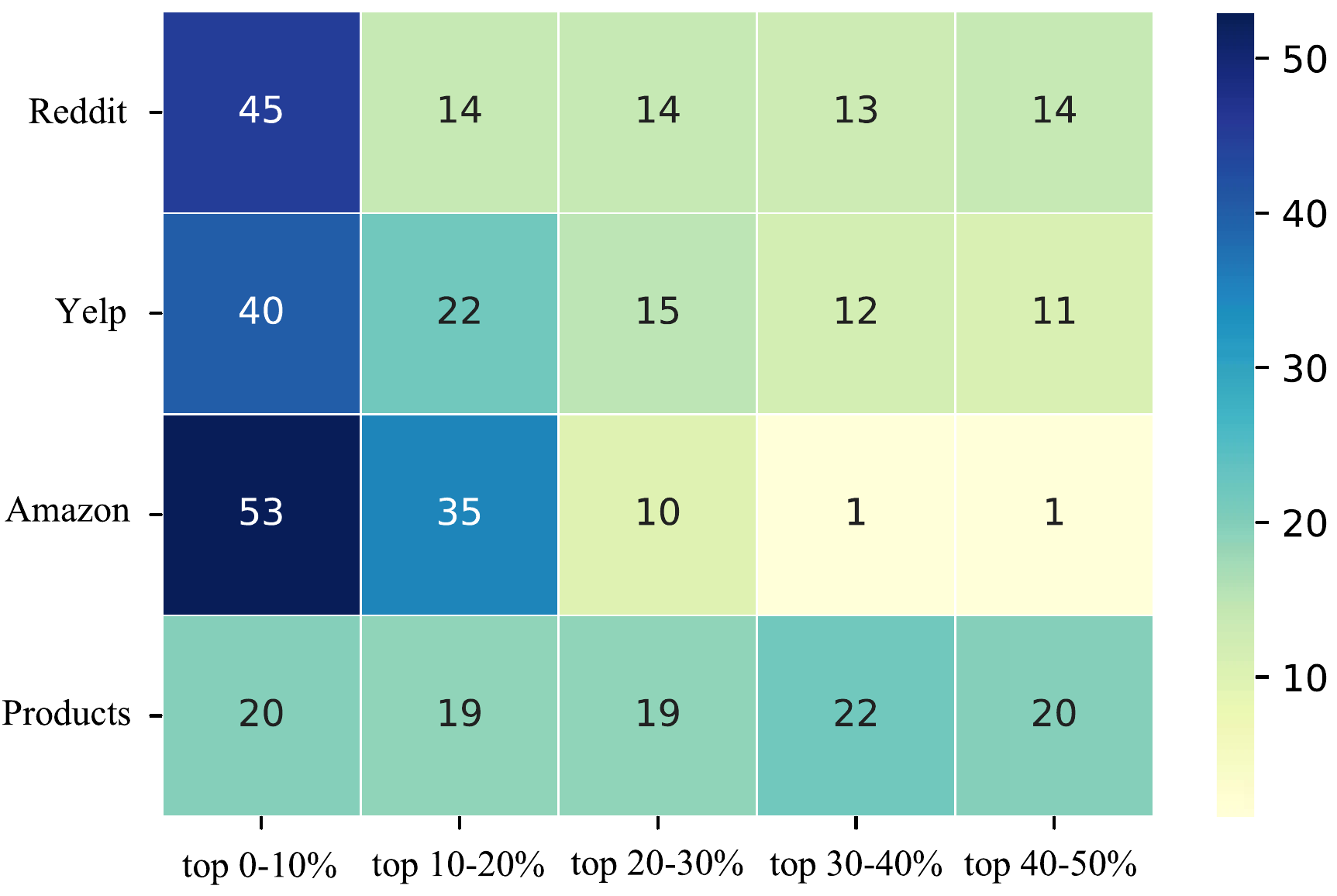}
\caption{Overlap between the \textit{top 10\%} of nodes with high WNH and the \textit{top 50\%} of nodes with large neighbors' amounts on four large-scale graphs. Please note that nodes are independently sorted by two metrics, i.e., WNH and the neighbors' amount. The value ``45'' in the top-left grid denotes that in Reddit, about 45\% of nodes with high WNH are simultaneously the \textit{top 10\%} of nodes with large neighbors' amounts.}
\label{fig4}
\end{figure}

\begin{proposition}
The weighted neighbor heterophily (WNH) is able to distinguish impacts made by individual neighbors compared to the naive neighbor heterophily.
\end{proposition}
\begin{proof}
We assume that label vectors of v and its two neighbors u and m, i.e., \textbf{c$_v$}, \textbf{c$_u$}, and \textbf{c$_m$} are $\mathrm{[0, 0, 1, 1, 0, 0]}$, $\mathrm{[1, 1, 0, 0, 0, 0]}$, and $\mathrm{[0, 0, 0, 0, 1, 1]}$, respectively. 2-Norm of (\textbf{c$_v$} - \textbf{c}$_u$) and (\textbf{c$_v$} - \textbf{c}$_m$) are equal, implying that impacts made by two individual neighbors cannot be distinguished. However, WNH introduces the linking probability associated with the edge between v and u (or m) to reflect a distinct impact.
\end{proof}

The proposed WNH is utilized to detect the neighbor heterophily for nodes in a large-scale graph. A higher WNH indicates that the neighboring distribution of the node is more heterophilic than a normal one. Therefore, a node with a higher WNH cannot be accurately classified even with a well-trained GNN. We remark that the computation of WNH for all nodes in a large-scale graph is offline performed before sampling, requiring no online resource.

\subsection{The Proposed Method: DropReef} \label{sec:DropReef}
Prior to the release of DropReef, we propose to detect the redundant nodes in a large-scale graph. In section \ref{sec:2.3}, we have conducted a preliminary experiment (given in Figure \ref{fig2}) to naively drop the \textit{top 5\%} of nodes on four large-scale graphs. We found the decline in the model accuracy is slight, and thus argued that the dropped HD nodes contain redundant information to a large extent. 

To support the argument, we conduct experiments to quantify the overlap between HD nodes and nodes with high WNH. For each dataset, we calculate WNH for all nodes and select the \textit{top 10\%} ones sorted by their WNH. Quantification of the overlap is given in Figure \ref{fig4}. Over 70\% of nodes with high WNH are simultaneously HD nodes, indicating that a node with a great number of neighbors will also have a heterophilic neighboring distribution with a high probability. 
Therefore, based on the quantification, we can define the data redundancy and redundant nodes in training large-scale graphs with GNNs.

\noindent \textbf{Definition 3} \textit{Redundancy in data denotes some information benefits little to the model performance (both training efficiency and model accuracy) but takes nontrivial cost for its computation. Therefore, redundant nodes can be further defined as HD nodes with large values in terms of WNH at the same time.}

\begin{algorithm}[t]
 \small
 
  \caption{\textbf{\normalsize\emph{DropReef (Offline Performed)}}} \label{alg1}
 $\texttt{INPUT}~$\emph{\textbf{Original Graph $\mathcal{G}(\mathcal{V}, \mathcal{E})$, Edge Prob. Matrix P,  Hyperparameters $\mathrm{TH}_{WNH}$, $\mathrm{TH}_{DEG}$}}; \\  
 $\texttt{OUTPUT}~$\emph{\textbf{Low-redundancy Graph $\mathcal{G}^{\prime}(\mathcal{V}^{\prime}, \mathcal{E}^{\prime})$}}; \\  

 $\mathcal{V}^{\prime} = \mathcal{V}$;\\
 $\mathcal{V}_{drop} = \varnothing$;\\      
 $\mathcal{V}_{tr} \leftarrow$ \emph{\textbf{get\_training\_nodes}}$(\mathcal{V})$;\\
 
 \For{$v \in \mathcal{V}_{tr}$}{ 
    $WNH_v = 0$;\\                         
    $\mathbf{c}_v \leftarrow$ \emph{\textbf{get\_label}}$(v)$;\\
    $D_{v} \leftarrow$ \emph{\textbf{get\_degree}}$(v)$;\\
    $N_{v} \leftarrow$ \emph{\textbf{get\_neighbors}}$(v)$;\\
    \For{$u \in N_{v}$}{
    $\mathbf{c}_u \leftarrow$ \emph{\textbf{get\_label}}$(u)$;\\
    $p_{vu} \leftarrow$ \emph{\textbf{get\_edge\_prob}}$(v, u)$;\\
    $\textit{WNH}_v = (\sum_u p_{vu} \times \Vert \mathbf{c}_v - \mathbf{c}_u \Vert_2)/D_v$ \tcp*{\blackcircled{1}}
    }
 }
    
    (Save the WNH for further usage and free the storage)

 \For{$v \in \mathcal{V}_{tr}$}{

    \If{$D_v \ge \mathrm{TH}_{DEG}~and~\textit{WNH}_v \ge \mathrm{TH}_{WNH}$}{
        $\mathcal{V}_{drop} \leftarrow \emph{\textbf{add\_node}}(\mathcal{V}_{del}, v)$ \tcp*{\blackcircled{2}}
    }
 }
    
\For{$v \in \mathcal{V}_{drop}$}{
    \For{$u \in N_v$}{
    $\mathcal{E}^{\prime} \leftarrow \emph{\textbf{remove\_edge}}(\mathcal{E}, v, u)$}
    $\mathcal{V}^{\prime} \leftarrow \emph{\textbf{remove\_node}}(v)$ \tcp*{\blackcircled{3}}
}
      
\end{algorithm}

Moreover, it is observed that there is a uniform distribution of high WNH nodes on Products. Nevertheless, the uniform distribution has no effect on dropping redundant nodes. We can adopt a flexible dropping process on all training nodes uniformly by adjusting hyperparameters.

Given two hyperparameters and a precomputed edge probability matrix \textit{\textbf{P}}, DropReef is offline performed on a large-scale graph to drop redundant nodes. A detailed process is given in Algorithm \ref{alg1}. DropReef consists of \textbf{three subprocesses}: computing metrics, detecting redundancy and dropping nodes.

\noindent \blackcircled{1} Computing metrics (i.e., WNH) for all training nodes requires all edge probabilities between training nodes and their neighbors. The probability matrix \textit{\textbf{P}} is precomputed by jointly using a graph auto-encoder (GAE) \cite{VGAE} and a batched probability predictor \cite{deepergcn}. Predicting edge probabilities in a batched manner can generally mitigate the storage consumption of a graph with tens of millions of edges \cite{deepergcn}.

\noindent \blackcircled{2} Detecting redundancy aims to select redundant nodes from all training nodes based on two hyperparameters, i.e., TH$_{WNH}$ and TH$_{DEG}$. Specifically, TH$_{WNH}$ is used to filter nodes with high WNH since such nodes are hard to classify accurately. TH$_{DEG}$ is a restriction on the minimum node degree, ensuring only nodes with a great number of neighbors will be considered to drop. According to two hyperparameters, nodes with high WNH and a great number of neighbors are regarded as redundant nodes.

\noindent \blackcircled{3} Dropping redundant nodes includes removing all edges associated with the redundant nodes and dropping these nodes from the training set. The output is a low-redundancy graph $\mathcal{G}^{\prime}$ that can be directly fed to GNNs for training.  

\noindent \textbf{Remark} We drop redundant nodes on the training set only. All subprocesses in DropReef are offline performed to generate a low-redundancy graph. Training such graphs with some popular GNNs will witness a considerable acceleration straightforwardly.

%% file: samples/5-exp.tex
\section{Experiment}
In this section, we conduct extensive experiments to demonstrate the effectiveness of DropReef. We give a detailed analysis based on the results and further discuss the adjustment of DropReef to offer a way for performing one's own DropReef. In addition, we quantify the sampled subgraphs and reveal the impact of DropReef on the subgraphs.

\subsection{Experimental Setting}

\tabcolsep 2.5pt

\begin{table}[t]
\renewcommand\arraystretch{1.5}
\centering
\caption{Statistics on large-scale graph datasets.} \label{tab1} 
\begin{tabular}{cccccccc}
\bottomrule
\textbf{Dataset} & \textbf{\#Node} & \textbf{\#Edge} & \textbf{Degree} & \textbf{Task Type} & \textbf{Train/Val/Test} \\ \hline
Reddit   & 232,965      & 11,606,919   & 50 & single-cls class.
&  0.66/0.10/0.24  \\
Yelp     & 716,847      & 6,977,410    & 10 & multi-cls class.    &  0.75/0.10/0.15  \\
Amazon   & 1,598,960    & 132,169,734  & 83 & multi-cls class.    &  0.85/0.05/0.10  \\
Products & 2,449,029    & 61,859,140   & 51 & multi-cls class.    &  0.08/0.02/0.90  \\
\bottomrule
\end{tabular}
\end{table}

\begin{figure*}[ht]
\centering
\includegraphics[width=1\textwidth]{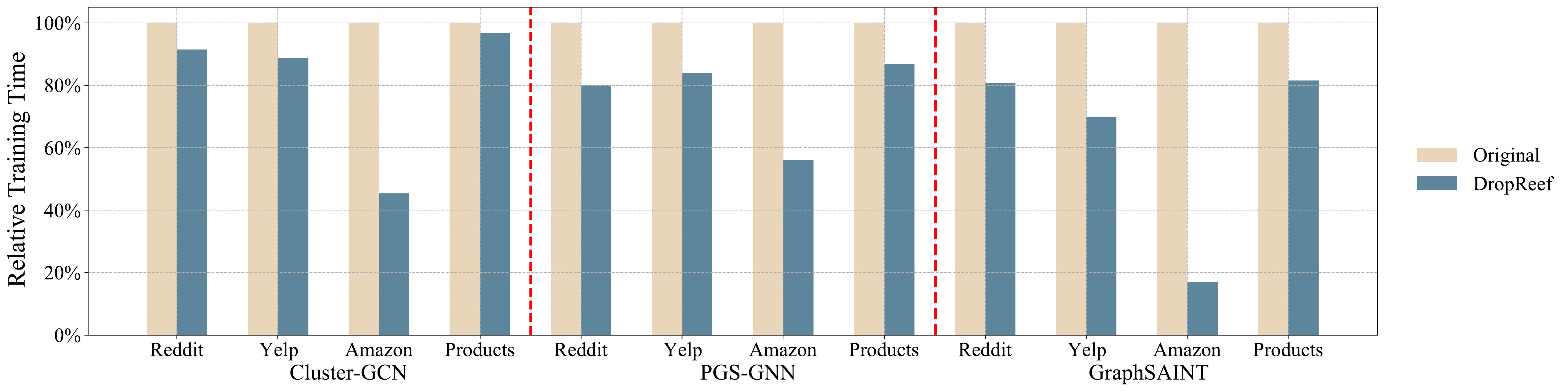}
\caption{Comparison on the training time among four large-scale graphs. We apply DropReef to three sampling-based GNNs to yield a speedup. Please note that \textbf{Original} denotes training the baselines, while \textbf{DropReef} denotes that DropReef was applied before training the baselines. In addition, training configurations between \textbf{Original} and \textbf{DropReef} are completely consistent to the official configurations of the corresponding baseline.}
\label{fig5}
\end{figure*}

\begin{table*}[t]
\centering
\caption{Comparison on the training time and the test accuracy between training original models and with DropReef equipped among four large-scale graphs.} \label{tab2} 
\begin{tabular}{ccrrrrrrrrr}
\bottomrule
\textbf{Model} & \textbf{Dataset} & 
\begin{tabular}[c]{@{}c@{}}\textbf{Training Time(s.)} \\ \textbf{(Original Case)} \end{tabular} &
\begin{tabular}[c]{@{}c@{}}\textbf{Training Time(s.)} \\ \textbf{(Use DropReef)} \end{tabular} &
\begin{tabular}[c]{@{}c@{}}\textbf{Time} \\ \textbf{Reduction} \end{tabular}  &
\begin{tabular}[c]{@{}c@{}}\textbf{Test Accuracy} \\ \textbf{(Original Case)} \end{tabular} &
\begin{tabular}[c]{@{}c@{}}\textbf{Test Accuracy} \\ \textbf{(Use DropReef)} \end{tabular} &
\begin{tabular}[c]{@{}c@{}} \textbf{Accuracy} \\ \textbf{Gap}\end{tabular} &
\begin{tabular}[c]{@{}c@{}} \textbf{Drop Node} \\ \textbf{Ratio}\end{tabular} &
\begin{tabular}[c]{@{}c@{}} \textbf{Drop Edge} \\ \textbf{Ratio}\end{tabular} 
\\ \bottomrule

\multirow{4}{*}{Cluster-GCN} 
 & Reddit    & 175.72  & \textbf{160.65} & \textbf{8.58\%}   & 0.9600  & \textbf{0.9604} & \textbf{$\uparrow$ 0.04\%} & 8.45\% & 29.02\% \\ 
 & Yelp      & 185.17  & \textbf{164.17} & \textbf{11.34\%}   & 0.6099  & \textbf{0.6158} & \textbf{$\uparrow$ 0.97\%} & 18.12\% & 65.12\% \\ 
 & Amazon    & 1400.37  & \textbf{605.04} & \textbf{56.79\%}   & 0.7566  & \textbf{0.7576} & \textbf{$\uparrow$ 0.13\%} & 5.00\% & 77.35\% \\ 
 & Products  & 161.39  & \textbf{156.14} & \textbf{3.26\%} &  0.7474  & \textbf{0.7516} & \textbf{$\uparrow$ 0.56\%} & 7.59\% & 23.46\%\\ \hline

\multirow{4}{*}{PGS-GNN} 
 & Reddit    & 45.80  & \textbf{36.71} & \textbf{19.86\%}   & 0.9549  & \textbf{0.9571} & \textbf{$\uparrow$ 0.02\%} & 5.73\% & 24.43\% \\ 
 & Yelp      & 831.26  & \textbf{697.17} & \textbf{16.13\%}   & 0.6256  & \textbf{0.6280} & \textbf{$\uparrow$ 0.38\%} & 4.54\% & 44.59\% \\ 
 & Amazon    & 4163.68  & \textbf{2338.21} & \textbf{43.84\%}   & 0.7759  & \textbf{0.7849} & \textbf{$\uparrow$ 1.15\%} & 5.00\% & 77.35\% \\ 
 & Products  & 618.75  & \textbf{536.44} & \textbf{13.30\%}   & 0.7389  & \textbf{0.7474} & \textbf{$\uparrow$ 1.15\%}  & 8.02\% & 16.26\%\\ \hline

\multirow{4}{*}{GraphSAINT} 
 & Reddit    & 36.79  & \textbf{29.73} & \textbf{19.19\%}   & 0.9615  & 0.9576 & \textbf{$\downarrow$} 0.41\% & 1.88\% & 9.12\% \\ 
 & Yelp      & 197.22  & \textbf{137.96} & \textbf{30.05\%}   & 0.6412  & 0.6358 & \textbf{$\downarrow$} 0.85\% & 3.05\% & 34.19\% \\ 
 & Amazon    & 1462.98  & \textbf{248.50} & \textbf{83.01\%}   & 0.7650  & \textbf{0.7787} & \textbf{$\uparrow$ 1.79\%} & 5.00\% & 77.35\% \\ 
 & Products  & 628.94  & \textbf{512.76} & \textbf{18.47\%}   & 0.7852  & \textbf{0.7898} & \textbf{$\uparrow$ 0.15\%} & 7.59\% & 23.46\% \\
\bottomrule
\end{tabular}
\end{table*}

\textbf{Datasets and Platform.}
We propose to demonstrate the effectiveness of DropReef in terms of efficiency and accuracy on the node classification task. We conduct experiments on four real-world graphs in large scale: Reddit \cite{graphsage}, Yelp \cite{graphsaint}, Amazon \cite{graphsaint}, and OGBN-Products (abbr. Products) \cite{OGB}. Among these four datasets, the smallest graph has more than two hundred thousand nodes and five million edges, while the largest has more than one and a half million nodes and one hundred million edges. Please see Table \ref{tab1} for more detailed statistics, where ``single-cls class.'' denotes the single-class classification task. All experiments are conducted on a Linux server equipped with dual 24-core Intel Xeon CPU E5-2650 v4 CPUs and an NVIDIA Tesla V100 GPU (16 GB memory).

\textbf{Baselines.} 
We apply DropReef to three popular sampling-based GNNs, i.e., Cluster-GCN \cite{clustergcn}, parallel graph sampling-based GNN (abbr. PGS-GNN) \cite{PGSampling}, and GraphSAINT \cite{graphsaint}, that have the capacity of training large-scale graphs. Specifically, we choose the random node sampler for GraphSAINT to perform sampling. In addition to the hyperparameters of DropReef, all experiments use official training configurations, e.g., sampling size, batch size, and learning rate. We remark that comparisons between training original baselines and with DropReef equipped are conducted under completely consistent configurations other than the input graph since the input graph was processed for dropping redundant nodes by DropReef.

\subsection{Experimental Results}

\begin{figure*}[htbp]
\centering
	\subfigure[Quantification on a sampled subgraph (vanilla case). Note that x and y axes are node indices.]{\includegraphics[width = 0.33\textwidth]{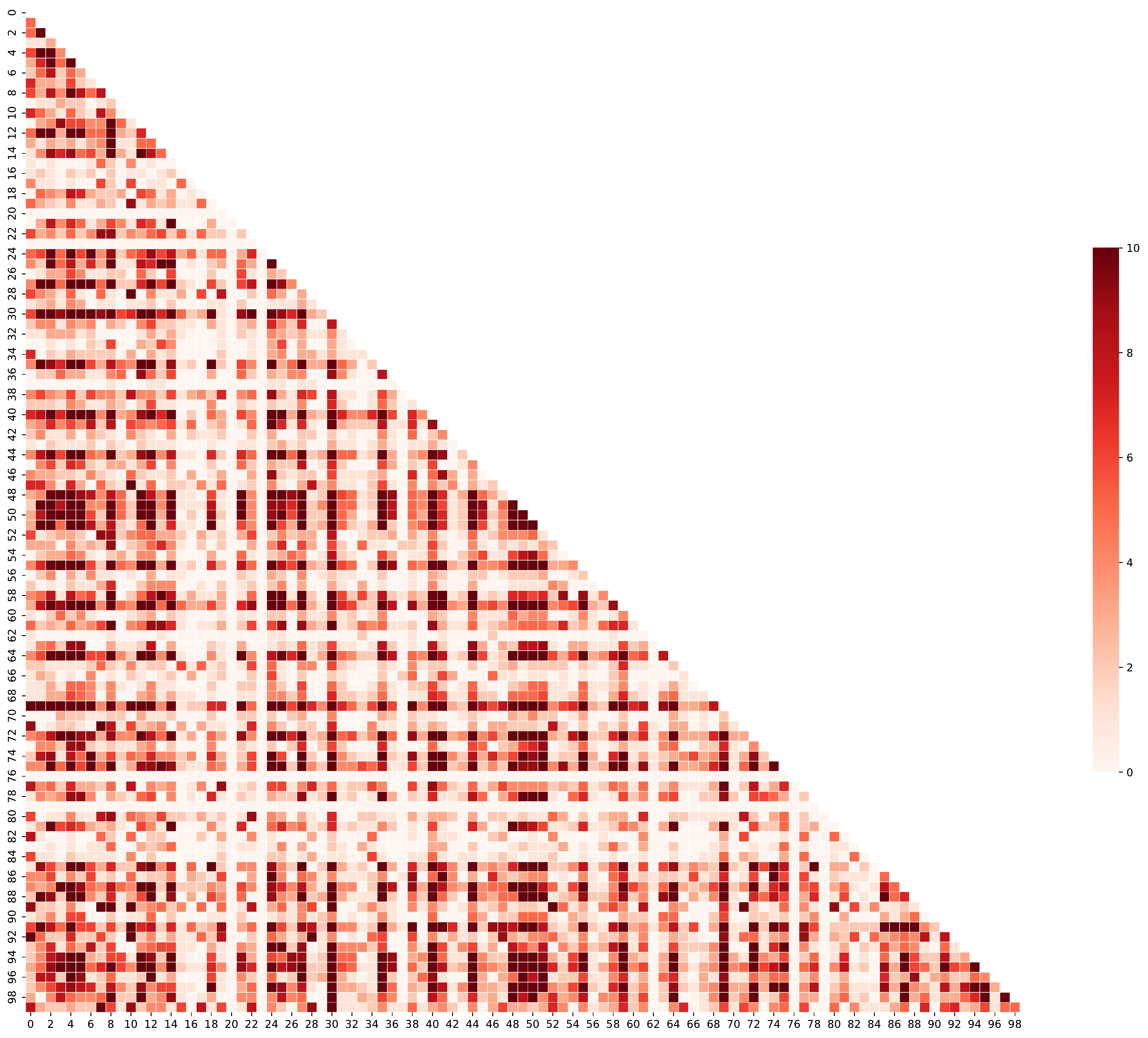}}
	\hspace{3mm}
	\subfigure[Quantification on a sampled subgraph (apply DropReef). Note that x and y axes are node indices.]{\includegraphics[width = 0.33\textwidth]{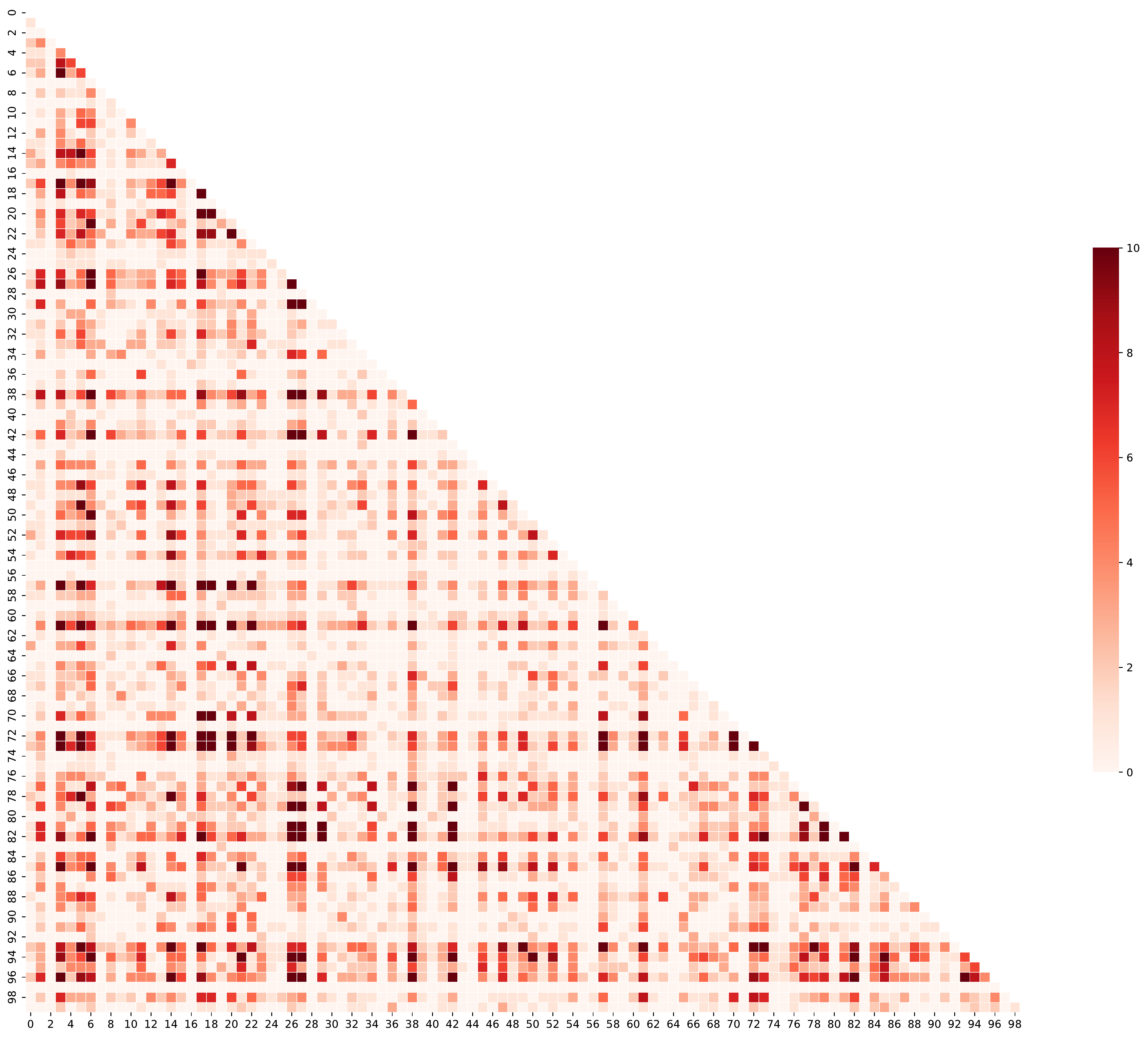}}
	\hspace{1mm}
	\subfigure[Number of shared neighbors in 3$\times$3 node regions in subgraphs.]{\includegraphics[width = 0.24\textwidth]{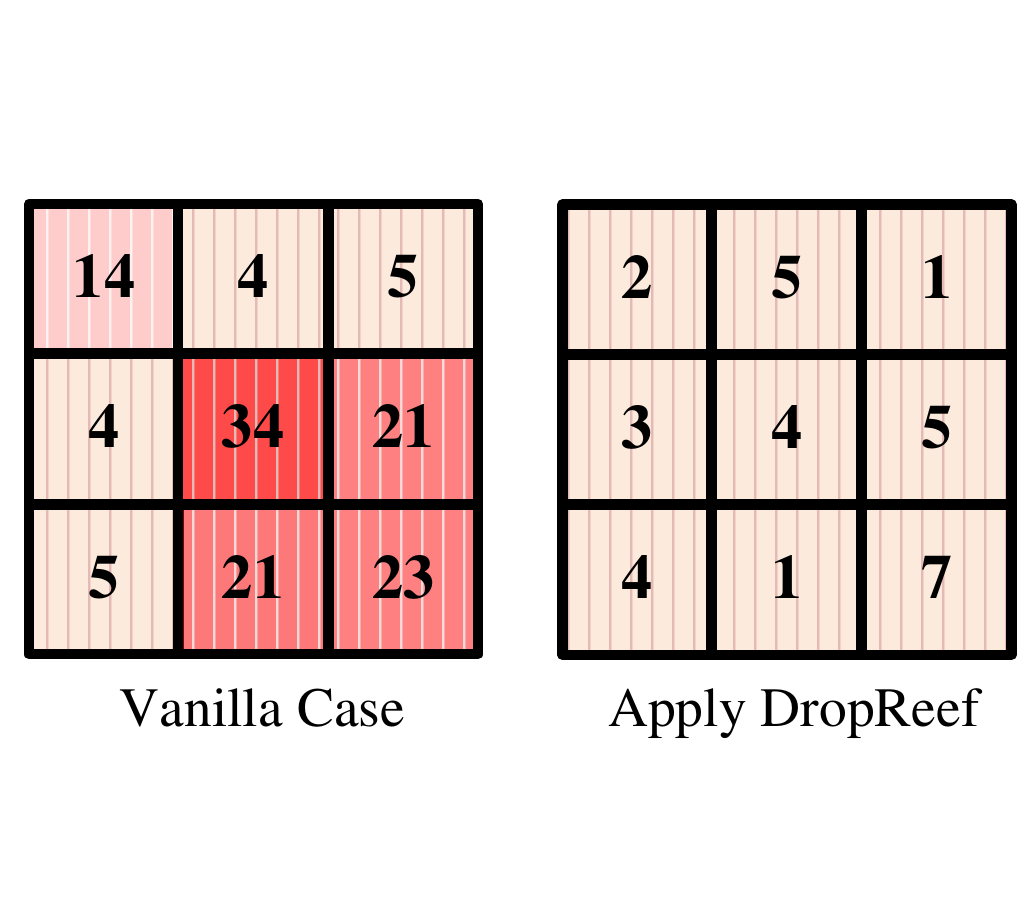}} 
\caption{Quantification of the shared neighbors between two nodes in a graph. A deeper color means a larger number of shared neighbors within a node region.}
\label{fig6}
\end{figure*}

We demonstrate the effectiveness of DropReef on accelerating the training through a comparison of the relative training time as shown in Figure \ref{fig5}. Please note that the training process consists of sampling and pure model training since they are typically mixed in a batched computing process. With the assistance of DropReef, redundant nodes are dropped from the training set of all large-scale graphs, thus yielding a remarkable acceleration in terms of training. Overall, applying DropReef helps reduce on average 26.80\% and up to 83.01\% of the training time. For each baseline, the (training) time reduction varies from approx. 10\% to 80\% among four large-scale graphs. Notably, the time reduction on Amazon dataset is significant. We notice that the time reduction on Products dataset among three baselines is less than on other graphs (given that the size of Products dataset is large) since only 8\% of nodes are split as the training set. A detailed performance is given in Table \ref{tab2}. Experiments are repeated three times to yield the average value.
Since DropReef removed redundant nodes from the training set, we argue that a larger training set of a graph can generally result in a greater gain from DropReef.

With respect to the model accuracy, the reduction of the training time does not bring about a drastic decline in the test accuracy. On the contrary, the test accuracy on most large-scale graphs is promoted after dropping redundancy, which can be attributed to the fact that the dropped redundant nodes are of no benefit to the model accuracy. We remark that DropReef removed redundant nodes from the training set only, leaving the validation and test sets unprocessed. Performance in terms of the time reduction and the test accuracy thereby demonstrates the effectiveness of DropReef in both efficiency and accuracy. Please note that Drop Node Ratio denotes the ratio of the number of dropped nodes to the number of all training nodes, and Drop Edge Ratio denotes the ratio of the number of the removed edges associated with the dropped nodes to the number of edges associated with all training nodes. We have the two metrics discussed in section \ref{sec:discuss}.

\subsection{Impact on Subgraphs}

DropReef has promoted the training efficiency by dropping redundant nodes from the training set, after which sampling and model training is performed on the low-redundancy graph. Since the batched training is conducted on sampled subgraphs, we thereby analyze the impact of DropReef on sampled subgraphs. Given the stochasticity of sampling, sampled subgraphs inevitably include dense node regions in a vanilla case. Such regions introduce considerable training time and are supposed to be dropped by DropReef. 

\begin{wrapfigure}{l}{4.8cm} 
\includegraphics[width=4.5cm]{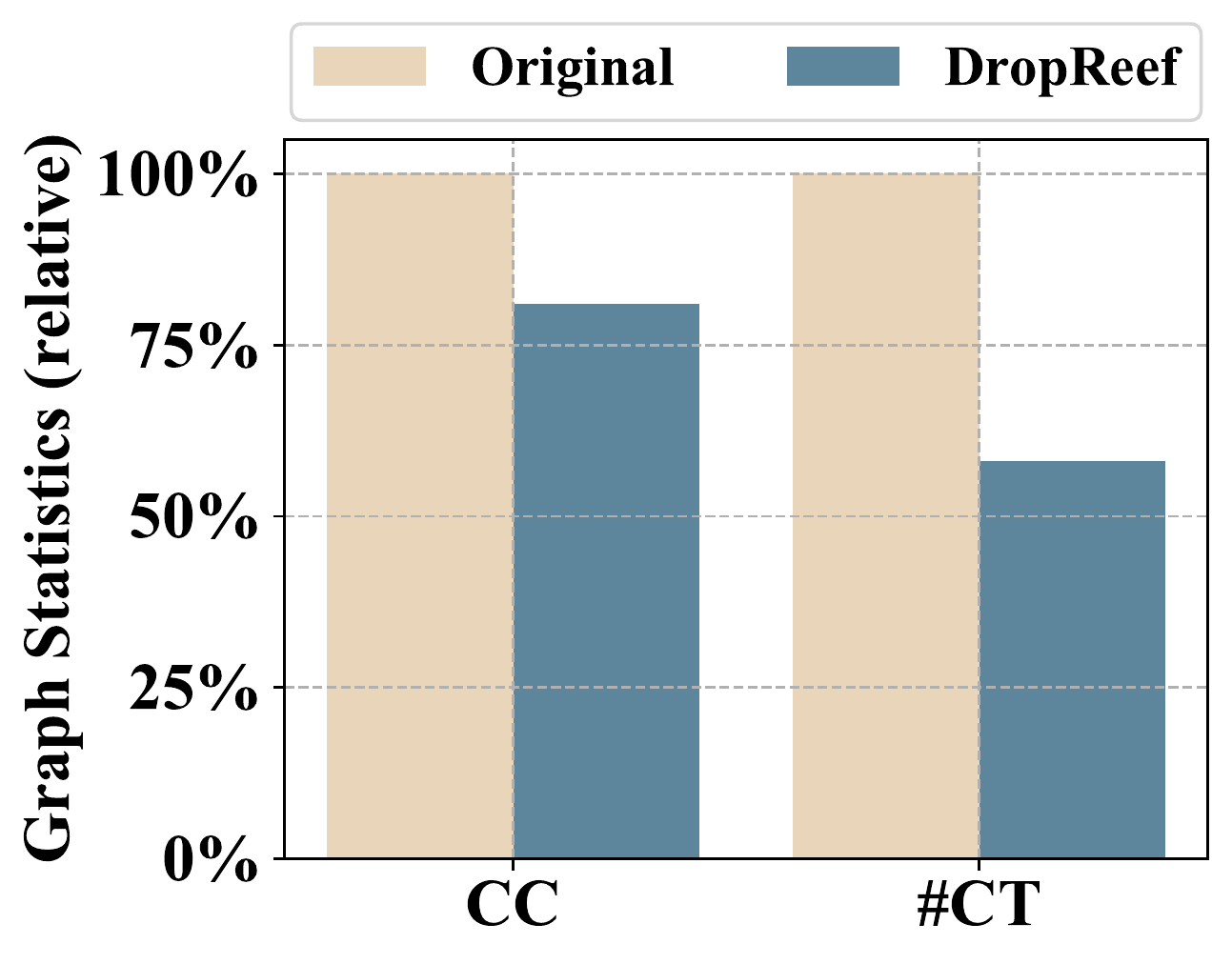}
\caption{Quantization on graph statistics.}
\label{fig7}
\end{wrapfigure}

To analyze a subgraph, we first conduct sampling on Amazon via a random node sampler provided in GraphSAINT. We sample two subgraphs in the same size for a vanilla case (regular sampling) and our case (apply DropReef before sampling) analysis. 
For each sampled subgraph, we compute the number of shared neighbors for any two nodes in the subgraph. For instance, given two nodes $v$ and $u$, n$_{vu}$ denotes the number of shared neighbors between $v$ and $u$. A large n$_{vu}$ indicates that a dense node region exists in the neighboring area of $v$ and $u$. As illustrated in Figure \ref{fig6} (a)\&(b), after applying DropReef, n$_{vu}$ in the sampled subgraph has a more balanced distribution than a vanilla case. The number of dense node regions in the sampled low-redundancy subgraph is less than in the vanilla subgraph, as suggested by two exemplars of quantified results of subgraphs in two cases (given in Figure \ref{fig6} (c)). We also quantify graph statistics, i.e., clustering coefficient (abbr. CC) and the number of closed triads (abbr. \#CT), of sampled subgraphs in two cases via Stanford Network Analysis Platform (SNAP) \cite{SNAP}. CC is used to measure the cliquishness of a typical neighborhood in a graph \cite{CC}. \#CT is used to reveal the number of potential dense connections in a node region composed of triad structures (among three nodes with any two of them connected) \cite{NTC}. We compute two statistics (average) based on 1000 sampled subgraphs in two cases. As illustrated in Figure \ref{fig7}, the sampled low-redundancy subgraphs have a more balanced distribution than vanilla ones from the graph statistical perspective.

\subsection{Discussion} \label{sec:discuss}
In addition to the size of the training set, another factor that affects the time reduction is the hyperparameters used in DropReef. Drop Node Ratio and Drop Edge Ratio are metrics controlled by these hyperparameters. As previously introduced in section \ref{sec:DropReef}, DropReef detects redundant nodes with the help of two thresholds, i.e., TH$_{WNH}$ and TH$_{DEG}$. TH$_{WNH}$ is used to select nodes with high WNH, while TH$_{DEG}$ is used to select nodes with a great number of neighbors. Nodes concurrently exceeding the two thresholds are regarded as redundant nodes. With respect to the adjustment of the two thresholds, we have analyzed the upper bound of $Hete_v$ in Proposition \ref{prop1}. WNH can be viewed as a weighted version of $Hete_v$ and will not exceed the upper bound of $Hete_v$. Moreover, we have quantified the average number of neighbors for each quintile of these HD nodes in Figure \ref{fig1}. The average degree of a graph is also a key value to construct the bound for TH$_{DEG}$. One can perform a unique DropReef on large-scale graphs by adjusting the hyperparameters based on the provided materials. We remark that turning down the two thresholds by a big margin can yield a more significant acceleration of the training time but will result in an undesirable decline in the model accuracy. Therefore, adjusting hyperparameters has been a trade-off between time and accuracy.

%% file: samples/6-relatedW.tex
\section{Related Work}
Related work of the proposed DropReef contains three main veins, i.e., graph sampling, graph sparsification, and dropout. 

\textbf{Graph Sampling:}
Graph sampling methods are popularly utilized to improve the training efficiency of GNNs. These methods sample nodes and construct a subgraph in each mini-batch, making it possible that the batch-performed GNN training takes less computation and storage cost than the full-batch training. 
Typically, graph sampling methods designed to train homogeneous graphs can be classified into three categories \cite{survey3}, i.e., node-wise, layer-wise, and subgraph-based sampling methods. Node-wise sampling methods, such as GraphSAGE \cite{graphsage}, sample a fixed fraction of neighbors for each node in a graph in a random manner. They reduce the training cost by restricting the sampling size. Layer-wise sampling methods, e.g., FastGCN \cite{fastgcn}, AS-GCN \cite{asgcn}, and LADIES \cite{LADIES}, simultaneously sample a fixed number of nodes by layers to train a layered GNN in a top-down manner. As the improvement of node-wise sampling methods, layer-wise sampling methods omit the focus on the neighborhoods of one single node in a graph, at the same time, perform according to the pre-computed sampling probability. Subgraph-based sampling methods, e.g., Cluster-GCN \cite{clustergcn} and GraphSAINT \cite{graphsaint}, sample subgraphs in each mini-batch by using graph partition algorithms or conditionally sampled nodes and edges, extending GNN training to large-scale graphs. DropReef is flexible to be equipped with mainstream sampling-based GNNs, boosting the performance of existing state-of-the-art models. Considering the drastically growing cost of training large-scale graphs, DropReef is capable of improving the efficiency and even the accuracy of sampling-based GNNs.

\textbf{Graph Sparsification:} 
In the GNN-related domain, graph sparsification methods propose to remove particular edges in a graph to improve the model accuracy or reduce the redundant computation. Typically, DropEdge \cite{dropedge} randomly removes edges to avoid over-fitting and over-smoothing issues. AdaptiveGCN \cite{adaptivegcn} and NeuralSparse \cite{NeuralSparse} use learnable models to remove task-irrelevant edges. GAUG \cite{GAUG} and PTDNet \cite{learningTOdrop} predict the quality of edges for the purpose of denoising, with neural networks leveraged. The above insightful efforts provide the promotion in terms of model accuracy. Moreover, UGS \cite{UGS} prunes both edges and the model weight to yield speedup in GNN inference. FastGAT \cite{fastgat} proposes to sparsify a graph to reduce attention coefficients' amount for attention-based GNNs by significantly removing useless edges. DropReef primarily distinguishes from graph sparsification methods as we drop redundant nodes rather than a considerable amount of edges in a graph. In addition, DropReef is a once-for-all procedure that requires no online resource to train a special model for sparsifying the graph. 
Furthermore, most existing graph sparsification methods generally conduct experiments on small datasets, while DropReef aims for the acceleration of training large-scale graphs. 

\textbf{Dropout:} In the deep learning domain, Dropout is proposed to randomly omit hidden units in a neural network layer for mitigating the over-smoothing issue \cite{Dropout1}. As an easy-to-use technique, Dropout is popularly utilized in diverse neural network models and achieves significant improvements compared to other regularization methods \cite{Dropout2}. DropReef can be considered a variant of Dropout customized for the graph domain. In our case, the objective for dropout is redundant nodes, which is a unique component of the graph data, compared to hidden units in a vanilla case. Moreover, the major function of Dropout is to mitigate the over-smoothing issue in the model training, while DropReef has succeeded in achieving a two-fold function, i.e., both the high efficiency and accuracy are well ensured.

%% file: samples/7-conclusion.tex
\section{Conclusion}
In this paper, we propose a once-for-all method, termed DropReef, to detect and drop the redundancy in large-scale graphs. By offline performing DropReef, redundant nodes together with their associated edges are dropped from the training set of a graph. We verify the effectiveness of DropReef using three popular sampling-based GNNs that have the capacity of training large-scale graphs. Experimental results demonstrate that DropReef succeeds in improving the efficiency of training large-scale graphs by GNNs, through dropping redundancy. As the scale of real-world graph data rapidly grows by the day, DropReef can be opportunely used to transform a large-scale graph into a low-redundancy one, benefiting the state-of-the-art methods designed for training large-scale graphs to a significant extent.

%% file: Main-Conference.bbl
\begin{thebibliography}{10}
\providecommand{\url}[1]{#1}
\csname url@samestyle\endcsname
\providecommand{\newblock}{\relax}
\providecommand{\bibinfo}[2]{#2}
\providecommand{\BIBentrySTDinterwordspacing}{\spaceskip=0pt\relax}
\providecommand{\BIBentryALTinterwordstretchfactor}{4}
\providecommand{\BIBentryALTinterwordspacing}{\spaceskip=\fontdimen2\font plus
\BIBentryALTinterwordstretchfactor\fontdimen3\font minus
  \fontdimen4\font\relax}
\providecommand{\BIBforeignlanguage}[2]{{%
\expandafter\ifx\csname l@#1\endcsname\relax
\typeout{** WARNING: IEEEtran.bst: No hyphenation pattern has been}%
\typeout{** loaded for the language `#1'. Using the pattern for}%
\typeout{** the default language instead.}%
\else
\language=\csname l@#1\endcsname
\fi
#2}}
\providecommand{\BIBdecl}{\relax}
\BIBdecl

\bibitem{scarselli2008graph}
F.~Scarselli, M.~Gori, A.~C. Tsoi, M.~Hagenbuchner, and G.~Monfardini, ``The
  graph neural network model,'' \emph{IEEE transactions on neural networks},
  vol.~20, no.~1, pp. 61--80, 2008.

\bibitem{graphsage}
W.~L. Hamilton, R.~Ying, and et~al., ``Inductive representation learning on
  large graphs,'' in \emph{Proceedings of the 31st International Conference on
  Neural Information Processing Systems}, 2017, pp. 1025--1035.

\bibitem{gnn_app2}
Z.~Wu, S.~Pan, and et~al., ``Connecting the dots: Multivariate time series
  forecasting with graph neural networks,'' in \emph{Proceedings of the 26th
  ACM SIGKDD International Conference on Knowledge Discovery \& Data Mining},
  2020, pp. 753--763.

\bibitem{gnn_app4}
M.~Schlichtkrull, T.~N. Kipf, and et~al., ``Modeling relational data with graph
  convolutional networks,'' in \emph{European semantic web conference}.\hskip
  1em plus 0.5em minus 0.4em\relax Springer, 2018, pp. 593--607.

\bibitem{gnn_app_security}
Y.~Liu, Z.~Li, S.~Pan, C.~Gong, C.~Zhou, and G.~Karypis, ``Anomaly detection on
  attributed networks via contrastive self-supervised learning,'' \emph{IEEE
  transactions on neural networks and learning systems}, vol.~33, no.~6, pp.
  2378--2392, 2021.

\bibitem{pope2019explainability}
P.~E. Pope, S.~Kolouri, M.~Rostami, C.~E. Martin, and H.~Hoffmann,
  ``Explainability methods for graph convolutional neural networks,'' in
  \emph{Proceedings of the IEEE/CVF Conference on Computer Vision and Pattern
  Recognition}, 2019, pp. 10\,772--10\,781.

\bibitem{zhang2022trustworthy}
H.~Zhang, B.~Wu, X.~Yuan, S.~Pan, H.~Tong, and J.~Pei, ``Trustworthy graph
  neural networks: Aspects, methods and trends,'' \emph{arXiv preprint
  arXiv:2205.07424}, 2022.

\bibitem{kipf2016semi}
T.~N. Kipf and M.~Welling, ``Semi-supervised classification with graph
  convolutional networks,'' in \emph{International Conference on Learning
  Representations {ICLR} 2017}, 2017.

\bibitem{velickovic2018graph}
P.~Veli{\v{c}}kovi{\'{c}}, G.~Cucurull, A.~Casanova, A.~Romero, P.~Li{\`{o}},
  and Y.~Bengio, ``Graph attention networks,'' in \emph{International
  Conference on Learning Representations {ICLR} 2018}, 2018.

\bibitem{xu2018how}
K.~Xu, W.~Hu, J.~Leskovec, and S.~Jegelka, ``How powerful are graph neural
  networks?'' in \emph{International Conference on Learning Representations
  {ICLR} 2019}, 2019.

\bibitem{clustergcn}
W.-L. Chiang, X.~Liu, S.~Si, Y.~Li, S.~Bengio, and C.-J. Hsieh, ``Cluster-gcn:
  An efficient algorithm for training deep and large graph convolutional
  networks,'' in \emph{Proceedings of the 25th ACM SIGKDD International
  Conference on Knowledge Discovery \& Data Mining}, 2019, pp. 257--266.

\bibitem{survey1}
S.~Abadal, A.~Jain, R.~Guirado, J.~L{\'o}pez-Alonso, and E.~Alarc{\'o}n,
  ``Computing graph neural networks: A survey from algorithms to
  accelerators,'' \emph{ACM Computing Surveys (CSUR)}, vol.~54, no.~9, pp.
  1--38, 2021.

\bibitem{survey2}
X.~Liu, M.~Yan, L.~Deng, G.~Li, X.~Ye, D.~Fan, S.~Pan, and Y.~Xie, ``Survey on
  graph neural network acceleration: An algorithmic perspective,'' \emph{arXiv
  preprint arXiv:2202.04822}, 2022.

\bibitem{fastgcn}
J.~Chen, T.~Ma, and C.~Xiao, ``Fastgcn: Fast learning with graph convolutional
  networks via importance sampling,'' in \emph{International Conference on
  Learning Representations}, 2018.

\bibitem{asgcn}
W.~Huang, T.~Zhang, Y.~Rong, and J.~Huang, ``Adaptive sampling towards fast
  graph representation learning,'' \emph{Advances in Neural Information
  Processing Systems}, vol.~31, pp. 4558--4567, 2018.

\bibitem{LADIES}
D.~Zou, Z.~Hu, and et~al., ``Layer-dependent importance sampling for training
  deep and large graph convolutional networks,'' \emph{Advances in neural
  information processing systems}, vol.~32, 2019.

\bibitem{RWT}
J.~Bai and et~al., ``Ripple walk training: A subgraph-based training framework
  for large and deep graph neural network,'' in \emph{2021 International Joint
  Conference on Neural Networks}.\hskip 1em plus 0.5em minus 0.4em\relax IEEE,
  2021, pp. 1--8.

\bibitem{graphsaint}
H.~Zeng, H.~Zhou, A.~Srivastava, R.~Kannan, and V.~Prasanna, ``{GraphSAINT}:
  Graph sampling based inductive learning method,'' in \emph{International
  Conference on Learning Representations}, 2020.

\bibitem{AC-sampling}
A.~Hasanzadeh, E.~Hajiramezanali, S.~Boluki, M.~Zhou, N.~Duffield,
  K.~Narayanan, and X.~Qian, ``Bayesian graph neural networks with adaptive
  connection sampling,'' in \emph{International conference on machine
  learning}.\hskip 1em plus 0.5em minus 0.4em\relax PMLR, 2020, pp. 4094--4104.

\bibitem{MV-sampling}
W.~Cong, R.~Forsati, and et~al., ``Minimal variance sampling with provable
  guarantees for fast training of graph neural networks,'' in \emph{Proceedings
  of the 26th ACM SIGKDD International Conference on Knowledge Discovery \&
  Data Mining}, 2020, pp. 1393--1403.

\bibitem{Bandit-sampling}
Z.~Liu, Z.~Wu, Z.~Zhang, J.~Zhou, S.~Yang, L.~Song, and Y.~Qi, ``Bandit
  samplers for training graph neural networks,'' \emph{Advances in Neural
  Information Processing Systems}, vol.~33, pp. 6878--6888, 2020.

\bibitem{Biased-sampling}
Q.~Zhang, D.~Wipf, Q.~Gan, and L.~Song, ``A biased graph neural network sampler
  with near-optimal regret,'' \emph{Advances in Neural Information Processing
  Systems}, vol.~34, 2021.

\bibitem{PGSampling}
H.~Zeng, H.~Zhou, and et~al., ``Accurate, efficient and scalable graph
  embedding,'' in \emph{2019 IEEE International Parallel and Distributed
  Processing Symposium (IPDPS)}.\hskip 1em plus 0.5em minus 0.4em\relax IEEE,
  2019, pp. 462--471.

\bibitem{smalldataset}
P.~Sen, G.~Namata, M.~Bilgic, L.~Getoor, B.~Galligher, and T.~Eliassi-Rad,
  ``Collective classification in network data,'' \emph{AI magazine}, vol.~29,
  no.~3, pp. 93--93, 2008.

\bibitem{Pagraph}
Z.~Lin, C.~Li, Y.~Miao, Y.~Liu, and Y.~Xu, ``Pagraph: Scaling gnn training on
  large graphs via computation-aware caching,'' in \emph{Proceedings of the
  11th ACM Symposium on Cloud Computing}, 2020, pp. 401--415.

\bibitem{survey3}
X.~Liu, M.~Yan, and et~al., ``Sampling methods for efficient training of graph
  convolutional networks: A survey,'' \emph{IEEE/CAA Journal of Automatica
  Sinica}, vol.~9, no.~2, pp. 205--234, 2021.

\bibitem{robustGNN1}
L.~Wang, W.~Yu, W.~Wang, W.~Cheng, W.~Zhang, H.~Zha, X.~He, and H.~Chen,
  ``Learning robust representations with graph denoising policy network,'' in
  \emph{2019 IEEE International Conference on Data Mining (ICDM)}.\hskip 1em
  plus 0.5em minus 0.4em\relax IEEE, 2019, pp. 1378--1383.

\bibitem{redundancy-GNN1}
Z.~Jia, S.~Lin, R.~Ying, J.~You, J.~Leskovec, and A.~Aiken, ``Redundancy-free
  computation for graph neural networks,'' in \emph{Proceedings of the 26th ACM
  SIGKDD International Conference on Knowledge Discovery \& Data Mining}, 2020,
  pp. 997--1005.

\bibitem{robustGNN2}
D.~Luo, W.~Cheng, W.~Yu, B.~Zong, J.~Ni, H.~Chen, and X.~Zhang, ``Learning to
  drop: Robust graph neural network via topological denoising,'' in
  \emph{Proceedings of the 14th ACM International Conference on Web Search and
  Data Mining}, 2021, pp. 779--787.

\bibitem{zhang2022understanding}
H.~Zhang, J.~Reed, H.~Ritter, Z.~DeVito, Z.~Yu, T.~Karaletsos, H.~He, G.~Dai,
  G.~Huang, A.~Ussery \emph{et~al.}, ``Understanding gnn computational graph: A
  coordinated computation, io, and memory perspective,'' \emph{Proceedings of
  Machine Learning and Systems}, vol.~4, 2022.

\bibitem{robustGNN3}
W.~Jin, Y.~Ma, X.~Liu, X.~Tang, S.~Wang, and J.~Tang, ``Graph structure
  learning for robust graph neural networks,'' in \emph{Proceedings of the 26th
  ACM SIGKDD International Conference on Knowledge Discovery \& Data Mining},
  2020, pp. 66--74.

\bibitem{robustGNN4}
N.~Entezari, S.~A. Al-Sayouri, A.~Darvishzadeh, and E.~E. Papalexakis, ``All
  you need is low (rank) defending against adversarial attacks on graphs,'' in
  \emph{Proceedings of the 13th International Conference on Web Search and Data
  Mining}, 2020, pp. 169--177.

\bibitem{survey4}
Z.~Wu, S.~Pan, F.~Chen, G.~Long, C.~Zhang, and S.~Y. Philip, ``A comprehensive
  survey on graph neural networks,'' \emph{IEEE transactions on neural networks
  and learning systems}, vol.~32, no.~1, pp. 4--24, 2020.

\bibitem{gnn_app3}
S.~Wan, C.~Gong, P.~Zhong, S.~Pan, G.~Li, and J.~Yang, ``Hyperspectral image
  classification with context-aware dynamic graph convolutional network,''
  \emph{IEEE Transactions on Geoscience and Remote Sensing}, vol.~59, no.~1,
  pp. 597--612, 2020.

\bibitem{gnn_prediction}
M.~Zhang and et~al., ``Link prediction based on graph neural networks,''
  \emph{Advances in neural information processing systems}, vol.~31, 2018.

\bibitem{OGB}
W.~Hu, M.~Fey, M.~Zitnik, Y.~Dong, H.~Ren, B.~Liu, M.~Catasta, and J.~Leskovec,
  ``Open graph benchmark: Datasets for machine learning on graphs,''
  \emph{Advances in neural information processing systems}, vol.~33, pp.
  22\,118--22\,133, 2020.

\bibitem{chen2013detecting}
Q.~Chen and et~al., ``Detecting local community structures in complex networks
  based on local degree central nodes,'' \emph{Physica A: Statistical Mechanics
  and its Applications}, vol. 392, no.~3, pp. 529--537, 2013.

\bibitem{clauset2005finding}
A.~Clauset, ``Finding local community structure in networks,'' \emph{Physical
  review E}, vol.~72, no.~2, p. 026132, 2005.

\bibitem{mcpherson2001birds}
M.~McPherson, L.~Smith-Lovin, and J.~M. Cook, ``Birds of a feather: Homophily
  in social networks,'' \emph{Annual review of sociology}, vol.~27, no.~1, pp.
  415--444, 2001.

\bibitem{Networks}
M.~Newman, \emph{Networks}.\hskip 1em plus 0.5em minus 0.4em\relax Oxford
  university press, 2018.

\bibitem{LGNN}
Z.~Chen, L.~Li, and J.~Bruna, ``Supervised community detection with line graph
  neural networks,'' in \emph{7th International Conference on Learning
  Representations, {ICLR} 2019,}.\hskip 1em plus 0.5em minus 0.4em\relax
  OpenReview.net, 2019.

\bibitem{beyond_homo}
J.~Zhu, Y.~Yan, L.~Zhao, M.~Heimann, L.~Akoglu, and D.~Koutra, ``Beyond
  homophily in graph neural networks: Current limitations and effective
  designs,'' \emph{Advances in Neural Information Processing Systems}, vol.~33,
  pp. 7793--7804, 2020.

\bibitem{MulC-imbalanced}
M.~Shi, Y.~Tang, and et~al., ``Multi-class imbalanced graph convolutional
  network learning,'' in \emph{Proceedings of the Twenty-Ninth International
  Joint Conference on Artificial Intelligence (IJCAI-20)}, 2020.

\bibitem{LinkPredSurvey}
M.~A. Hasan and M.~J. Zaki, ``A survey of link prediction in social networks,''
  in \emph{Social network data analytics}, 2011, pp. 243--275.

\bibitem{LinkPred_1}
D.~Liben-Nowell and J.~Kleinberg, ``The link-prediction problem for social
  networks,'' \emph{Journal of the American society for information science and
  technology}, vol.~58, no.~7, pp. 1019--1031, 2007.

\bibitem{LinkPred_2}
M.~Zhang and et~al., ``Link prediction based on graph neural networks,''
  \emph{Advances in neural information processing systems}, vol.~31, 2018.

\bibitem{VGAE}
T.~N. Kipf and M.~Welling, ``Variational graph auto-encoders,'' \emph{arXiv
  preprint arXiv:1611.07308}, 2016.

\bibitem{deepergcn}
G.~Li, C.~Xiong, A.~Thabet, and B.~Ghanem, ``Deepergcn: All you need to train
  deeper gcns,'' \emph{arXiv preprint arXiv:2006.07739}, 2020.

\bibitem{SNAP}
J.~Leskovec and R.~Sosi{\v{c}}, ``Snap: A general-purpose network analysis and
  graph-mining library,'' \emph{ACM Transactions on Intelligent Systems and
  Technology}, vol.~8, no.~1, p.~1, 2016.

\bibitem{CC}
D.~J. Watts and S.~H. Strogatz, ``Collective dynamics of
  ‘small-world’networks,'' \emph{nature}, vol. 393, no. 6684, pp. 440--442,
  1998.

\bibitem{NTC}
D.~Easley and J.~Kleinberg, \emph{Networks, crowds, and markets: Reasoning
  about a highly connected world}.\hskip 1em plus 0.5em minus 0.4em\relax
  Cambridge university press, 2010.

\bibitem{dropedge}
Y.~Rong, W.~Huang, T.~Xu, and J.~Huang, ``Dropedge: Towards deep graph
  convolutional networks on node classification,'' \emph{arXiv preprint
  arXiv:1907.10903}, 2019.

\bibitem{adaptivegcn}
D.~Li, T.~Yang, L.~Du, Z.~He, and L.~Jiang, ``Adaptivegcn: Efficient gcn
  through adaptively sparsifying graphs,'' in \emph{Proceedings of the 30th ACM
  International Conference on Information \& Knowledge Management}, 2021, pp.
  3206--3210.

\bibitem{NeuralSparse}
C.~Zheng, B.~Zong, W.~Cheng, and et~al., ``Robust graph representation learning
  via neural sparsification,'' in \emph{ICML}, 2020, pp. 11\,458--11\,468.

\bibitem{GAUG}
T.~Zhao, Y.~Liu, L.~Neves, O.~Woodford, M.~Jiang, and N.~Shah, ``Data
  augmentation for graph neural networks,'' \emph{arXiv preprint
  arXiv:2006.06830}, 2020.

\bibitem{learningTOdrop}
D.~Luo, W.~Cheng, W.~Yu, B.~Zong, J.~Ni, H.~Chen, and X.~Zhang, ``Learning to
  drop: Robust graph neural network via topological denoising,'' in
  \emph{Proceedings of the 14th ACM International Conference on Web Search and
  Data Mining}, 2021, pp. 779--787.

\bibitem{UGS}
T.~Chen, Y.~Sui, X.~Chen, A.~Zhang, and Z.~Wang, ``A unified lottery ticket
  hypothesis for graph neural networks,'' in \emph{International Conference on
  Machine Learning}.\hskip 1em plus 0.5em minus 0.4em\relax PMLR, 2021, pp.
  1695--1706.

\bibitem{fastgat}
R.~S. Srinivasa, C.~Xiao, L.~Glass, J.~Romberg, and J.~Sun, ``Fast graph
  attention networks using effective resistance based graph sparsification,''
  \emph{arXiv preprint arXiv:2006.08796}, 2020.

\bibitem{Dropout1}
G.~E. Hinton, N.~Srivastava, A.~Krizhevsky, I.~Sutskever, and R.~R.
  Salakhutdinov, ``Improving neural networks by preventing co-adaptation of
  feature detectors,'' \emph{arXiv preprint arXiv:1207.0580}, 2012.

\bibitem{Dropout2}
N.~Srivastava, G.~Hinton, A.~Krizhevsky, I.~Sutskever, and R.~Salakhutdinov,
  ``Dropout: a simple way to prevent neural networks from overfitting,''
  \emph{The journal of machine learning research}, vol.~15, no.~1, pp.
  1929--1958, 2014.

\end{thebibliography}
